\documentclass{article} %
\usepackage{times}  %
\usepackage{helvet}  %
\usepackage{courier}  %
\usepackage[hyphens]{url}  %
\usepackage{graphicx} %
\usepackage{natbib}  %
\usepackage{caption} %
\usepackage{fullpage}

\usepackage[utf8]{inputenc} %
\usepackage[T1]{fontenc}    %
\usepackage{hyperref}       %
\usepackage{booktabs}       %
\usepackage{amsfonts}       %
\usepackage{nicefrac}       %
\usepackage{microtype}      %
\usepackage{xcolor}         %

\usepackage{todonotes}
\usepackage{parskip}
\usepackage{amsfonts}
\usepackage{dsfont}
\usepackage{amsmath,amssymb,amsthm}
\usepackage{algorithm}
\usepackage[noend]{algpseudocode}
\usepackage{bbm}

\usepackage{threeparttable}
\usepackage{caption}
\usepackage{blindtext}

\newtheorem{theorem}{Theorem}[section]

\newtheorem{lemma}[theorem]{Lemma}

\newtheorem{remark}[theorem]{Remark}

\newtheorem{definition}[theorem]{Definition}

\newcommand{\remove}[1]{}

\newcommand{\ddim}{ddim}

\def\eps{\varepsilon}
\newcommand{\vc}{\mathrm{VC}}

\usepackage{bbm}

\newcommand{\beq}{\begin{eqnarray*}}
\newcommand{\eeq}{\end{eqnarray*}}
\newcommand{\beqn}{\begin{eqnarray}}
\newcommand{\eeqn}{\end{eqnarray}}
\newcommand{\ben}{\begin{enumerate}}
\newcommand{\een}{\end{enumerate}}

\newcommand{\Norm}[1]{\left\Vert #1 \right\Vert}

\newcommand{\R}{\mathbb{R}}

\newcommand{\N}{\mathbb{N}}

\newcommand{\abs}[1]{\left| #1 \right|}

\newcommand{\err}{\operatorname{err}}

\newcommand{\E}{\mathop{\mathbb{E}}}

\newcommand{\cubesize}{r}
\newcommand{\indicator}[1]{\mathbbm{1}\left[#1\right]}
\newcommand{\measure}{P}
\newcommand{\reg}{\eta}
\newcommand{\diam}[1]{diam\left(#1\right)}

\newcommand{\sample}{S}
\newcommand{\unlabeled}{U}
\newcommand{\domain}{\mathcal{X}}
\newcommand{\ourdomain}{[0,1]}
\newcommand{\alg}{\mathcal{A}}
\newcommand{\class}{\mathcal{C}}

\newcommand{\algname}[1]{\texttt{#1}}
\newcommand{\bigO}[1]{\mathcal{O}\left(#1\right)}
\newcommand{\Lap}{{\rm Lap}}
\newcommand{\distH}[2]{d_H(#1,#2)}
\newcommand{\metric}{\rho}
\newcommand{\dist}[2]{\metric(#1,#2)}
\newcommand{\GS}[1]{GS(#1)}
\newcommand\numberthis{\addtocounter{equation}{1}\tag{\theequation}}
\newcommand{\realcount}[1]{count_{#1}}
\newcommand{\dpcount}{\hat{c}}
\newcommand{\dplabels}{\hat{y}}
\newcommand{\partition}{\mathcal{C}}
\newcommand{\arbitrarilysmall}{\tau}
\newcommand{\support}{T}
\newcommand{\supportcube}{\mathcal{T}}
\newcommand{\radius}{r}
\newcommand{\packing}{\mathcal{N}}

\newcommand{\cover}{\mathcal{M}}
\newcommand{\coverset}{\hat{V}}
\newcommand{\voronoi}{V}
\newcommand{\vorsize}{N}
\newcommand{\vorsample}{A}
\newcommand{\ball}{T}
\newcommand{\dradius}{\theta}
\newcommand{\point}{p}
\newcommand{\expddim}{\lambda}

\title{Differentially-Private Bayes Consistency}

\author{
    Olivier Bousquet  
    \thanks{
    Google, Brain Team.
    \texttt{obousquet@google.com}
    } 
    \and
    Haim Kaplan 
    \thanks{
    Tel Aviv University and Google research. 
    \texttt{haimk@tau.ac.il.}
    } 
    \and
    Aryeh Kontorovich 
    \thanks{
    Ben-Gurion University of the Negev.
    \texttt{karyeh@cs.bgu.ac.il}
    } 
    \and
    Yishay Mansour 
    \thanks{
    Tel Aviv University and Google Research. 
    \texttt{mansour.yishay@gmail.com}
    } 
    \and
    Shay Moran 
    \thanks{
    Technion – Israel Institute of Technology and Google Research. 
    \texttt{smoran@technion.ac.il}
    } 
    \and
    Menachem Sadigurschi 
    \thanks{
    Ben-Gurion University of the Negev.
    \texttt{menisadi@gmail.com}
    }
    \and
    Uri Stemmer 
    \thanks{
    Tel Aviv University and Google Research.
    \texttt{u@uri.co.il}
    }
}

\date{December 8, 2022}

\begin{document}

\maketitle

\begin{abstract}
    We construct a universally Bayes consistent learning rule
        that satisfies differential privacy (DP).
        We first handle the setting of binary classification 
        and then extend our rule to the more 
        general setting of density estimation (with respect to the total variation metric).
The existence of a universally consistent DP learner 
    reveals a stark difference with the distribution-free PAC model. 
    Indeed, in the latter DP learning is extremely limited: 
    even one-dimensional linear classifiers
    are not privately learnable in this stringent model.
    Our result thus demonstrates that by allowing
    the learning rate to depend on the target distribution,
    one can circumvent the above-mentioned impossibility result
    and in fact learn \emph{arbitrary} distributions by a single DP algorithm.
As an application, we prove that any VC class can be privately learned in a semi-supervised
    setting with a near-optimal \emph{labeled} sample complexity of $\tilde O(d/\eps)$ labeled examples
    (and with an unlabeled sample complexity that can depend on the target distribution).
    
\end{abstract}
\section{Introduction}

Motivated by the increasing awareness and demand for user privacy, the line of work on {\em private learning} aims to construct learning algorithms that provide privacy protections for their training data. It is especially desirable to achieve {\em differential privacy} \citep{DMNS06}, a privacy notion which has been widely adopted by the academic community, government agencies, and big corporations like Google, Apple, and Microsoft. Intuitively, differential privacy requires that the outcome of the learner (e.g., the returned hypothesis) should leak almost no information on any single data point from the training set.
Formally, the definition of differential privacy is,
\begin{definition}\label{def:dpIntro} 
A randomized algorithm $\alg:X^n\rightarrow H$ is {\em $(\eps,\delta)$-differentially private (DP)} if for any two input datasets that differ on one point $S,S'\in X^n$ (such datasets are called {\em neighboring}) and for any event $F\subseteq H$ it holds that $\Pr[\alg(S)\in F]\leq e^\eps\cdot \Pr[\alg(S')\in F]+\delta.$ 
\end{definition}

Over the last few years we have witnessed an explosion of research on differential privacy in general and differentially private learning in particular. 
Nevertheless, despite tremendous efforts, designing effective differentially private tools for statistical and machine learning tasks has proven to be quite challenging. 
On the theory side, our current understanding of the possibilities and the limits of differentially private learning is far from being complete.  For example, only recently DP learnability has been characterized in the classical PAC model \citep{AlonLMM19,BunLM20}; and still, the best known bounds on sample complexity are absurdly loose. 

One possible explanation for these challenges is that most of the works on DP learning are inspired and explained by {worst-case} mathematical models such as the theory of Probably Approximately Correct (PAC) Learning \citep{valiant1984theory}, which is based on a {\em distribution-free perspective}. While it gives rise to a clean and compelling mathematical picture, one may argue that the PAC
model fails to capture at a fundamental level the true behavior of many practical learning problems (regardless of privacy consideration). A key criticism of the PAC model is that the distribution-independent definition of learnability is too pessimistic: real-world data is rarely worst-case, and experiments show that practical learning rates can be much faster than is predicted by PAC theory \citep{CohnT90,CohnT92}. It therefore appears that the worst-case nature of the PAC model hides key features that are observed in practice. Furthermore, these shortcomings seem to be amplified in the context of {\em private} PAC learning: even simple classes such as one-dimensional linear classifiers over the interval $\ourdomain$ are not learnable in this stringent model \cite{BunNSV15,AlonLMM19}, even though these classes are trivially learnable without privacy constraints. We believe that such impossibility results reflect the worst-case distribution-free nature of the PAC model rather than fundamental limitations of DP learning. We therefore advocate the study of distribution-dependent private-learning, as this can lead to a more optimistic (and realistic) landscape of differentially private learning.

\subsection{Making the setting explicit}
\label{sec:setting-explicit}
In this work we set out to explore DP learnability while relaxing the distribution-free requirement of the PAC model. That is, we ask what can be learned by DP algorithms if the learning rate can depend on the target distribution. Towards making our setting explicit, we start by recalling the (worst-case, distribution-free) PAC learning model and its distribution-dependent variants for the setting of binary classification. We remark that the presentation here is somewhat simplified. In particular, in order to allow for an easy comparison to the PAC learning model, here we state the definitions in terms of high probability bounds on the error rather than the more common expected error. See Section~\ref{sec:prelims} for the precise definitions.

\begin{definition}[The PAC model \citep{valiant1984theory}]\label{def:introPAC}
Let $C$ be a concept class over a domain $X$. An algorithm $\alg$ is a {\em PAC learner} for $C$ if for every $\alpha,\beta$ there is a constant $n=n(\alpha,\beta,C)$ such that the following holds {\em for every distribution $\measure$ over $X\times\{0,1\}$}.
$$
\Pr_{\substack{S\sim\measure^n\\h\leftarrow\alg(S)}}[\err_{\measure}(h)-\inf_{f\in C}\left\{\err_{\measure}(f)\right\}>\alpha]<\beta,
$$
where $\err_{\measure}(h)=\E_{(x,y)\sim \measure}[1[h(x)\neq y]]$.
\end{definition}

The above definition is {\em uniform} over $\measure$ in the sense that a sample size $n$ suffices {\em for every} underlying distribution $\measure$. This is very pessimistic, as it allows the worst-case distribution to change with the sample size. This arguably does not reflect the practice of machine learning: in a given learning scenario, the data distribution $\measure$ is fixed, while the learner is allowed to collect an arbitrary amount of data (depending on factors such as the desired accuracy and the available computational resources). This motivates the following variant of the definition, where we reverse the order of quantifiers such that the sample size $n$ can depend on the underlying distribution. Formally,

\begin{definition}[Universal learning, informal \citep{DBLP:conf/stoc/BousquetHMHY21}]\label{def:introDD}
Let $C$ be a concept class over a domain $X$. An algorithm $\alg$ is a {\em universal learner} for $C$ if for every $\alpha,\beta$ and for every distribution $\measure$ over $X\times\{0,1\}$ there is a constant $n=n(\alpha,\beta,C,\measure)$ such that the following holds.
$$
\Pr_{\substack{S\sim\measure^n\\h\leftarrow\alg(S)}}[\err_{\measure}(h)-\inf_{f\in C}\left\{\err_{\measure}(f)\right\}>\alpha]<\beta.
$$
\end{definition}

The term {\em universal} refers to the requirement that the learner succeeds for {\em every}
underlying distribution $\measure$, but not uniformly over all distributions. 
In fact, with this order of quantifiers we may also seek a single learning algorithm that learns {\em any} class $C$,  by taking $C=2^X$. Formally,

\begin{definition}[Universal consistent learning, informal \citep{doi:10.1098/rsta.1922.0009, devroye2013probabilistic}]\label{def:introUC}
Let $X$ be a domain. An algorithm $\alg$ is a {\em universal consistent (UC) learner} over $X$ if for every $\alpha,\beta$ and for every distribution $\measure$ over $X\times\{0,1\}$ there is a constant $n=n(\alpha,\beta,\measure)$ such that the following holds.
$$
\Pr_{\substack{S\sim\measure^n\\h\leftarrow\alg(S)}}[\err_{\measure}(h)-\inf_{f:X\rightarrow\{0,1\}}\left\{\err_{\measure}(f)\right\}>\alpha]<\beta.
$$
\end{definition}
Note, that this definition is in fact equivalent to classical definition of \emph{Bayes-consistency} (\citet{steinwart2008support, devroye2013probabilistic}), 
we chose to use this terminology to link it with the broad and modern context of \emph{universal learning}.

Of course, we are interested in learning algorithms (satisfying either Definition~\ref{def:introPAC} or~\ref{def:introDD} or~\ref{def:introUC}) that are also differentially private (i.e., algorithms that satisfy Definition~\ref{def:dpIntro} w.r.t.\ their training set).

\subsection{Our results}
\label{sec:our-results}

We begin by presenting a universal consistent learner for the bounded euclidean space $[0,1]^d$. Formally,

\begin{theorem}\label{thm:mainIntro}
For every $d\in\N$ and every $\eps\leq1$ there is an $(\eps,0)$-differentially private universal consistent (UC) learner over $\ourdomain^d$.
\end{theorem}

Recall that, as we mentioned, learning one-dimensional linear classifiers over $\ourdomain$ with differential privacy is impossible in the PAC model. Theorem~\ref{thm:mainIntro} circumvents this impossibility result: not only are one-dimensional linear classifiers learnable in the UC model, but in fact {\em every} class (over $\ourdomain^d$) is learnable in this setting, and furthermore, there is a single (universal consistent) algorithm that learns them all (w.r.t.\ any distribution).

To obtain Theorem~\ref{thm:mainIntro} we design a simple variant for the classical {\em histogram rule} \citep{Glick73,GordonO78,GordonO80,devroye2013probabilistic} that partitions $\ourdomain^d$ into cubes of the same size (where the size decreases with the sample size $n$), and makes a decision according to the majority vote within each cube. This algorithm is particularly suitable for differential privacy, and can be made private simply by adding noise to the votes within each cube. In the analysis, we show that this does not break the universal consistency of the histogram rule.

We then extend Theorem~\ref{thm:mainIntro} in three aspects: 
\begin{enumerate}
    \item We extend our results to the more general setting of {\em density estimation} %
    (with respect to the total variation metric).
    \item We extend our results to the {\em unbounded} euclidean space $\R^d$, and additionally, to metric spaces with {\em finite doubling dimension}.
    \item We present applications of our results, specifically in the context of semi-supervised learning.
\end{enumerate}

\subsubsection{Density estimation}
\label{subsec:intro-density}
We seek a differentially private algorithm satisfying the following definition.

\begin{definition}[Universal consistent density estimation, informal \citep{devroye1985nonparametric}]\label{def:introUCinformal}
Let $X$ be a domain and let $\alg$ be an algorithm whose output is a density function over $X$. Algorithm $\alg$ is a {\em universal consistent (UC) density estimator} over $X$ if for every $\alpha,\beta$ and for every distribution $\measure$ over $X$ there is a constant $n=n(\alpha,\beta,\measure)$ such that
$
\Pr_{\substack{S\sim\measure^n\\f\leftarrow\alg(S)}}[
{\rm TV}(f,\measure)
>\alpha]<\beta.
$
\end{definition}

Unlike our private UC learner, which satisfies differential privacy with $\delta=0$ (this is sometimes referred to a {\em pure} differential privacy), our private UC density estimator only satisfies differential privacy with $\delta>0$. When using $\delta>0$, it is commonly agreed that the definition of differential privacy only provides meaningful guarantees as long as $\delta\ll1/n$. That is, unlike with our private UC learner, where the privacy parameter $\eps$ is constant (independent of the sample size $n$), now we must let $\delta$ decay with $n$. Our result is the following.

\begin{theorem}\label{thm:introDE}
Let $d\in\N$, let $\eps\leq1$, and let $\delta:\N\rightarrow[0,1]$ be a function satisfying $\delta(n)=\omega(2^{-\sqrt{n}})$. There is an $(\eps,\delta(n))$-differentially private universal consistent (UC) density estimator over $\R^d$.
\end{theorem}

\begin{remark}
As an immediate corollary of Theorem~\ref{thm:introDE}, we also obtain a UC learner over the {\em unbounded} euclidean space $\R^d$. This, however, comes at the cost of guarantying approximate-privacy, rather than pure-privacy.
\end{remark}

\subsubsection{Finite doubling dimension.} 
We extend our results regarding universal consistent learning also to metric spaces with finite doubling dimension. These results, as well as the necessary preliminaries, are given in the supplementary material. Here we only state the result.

\begin{theorem}[informal]\label{thm:main-doubling}
There is a differentially private universal consistent learner for every separable metric space with finite doubling dimension.
\end{theorem}

We view our work as making an important first step towards understanding differentially private universal learning. Our work raises several interesting open questions, such as (1) understanding which {\em learning rates} can be achieved for private universal learning; (2) understanding how does computational efficiency affects the achievable rates; and (3) seeking a pure $(\eps,0)$-differentially private UC density estimator (or learner over unbounded spaces).

\paragraph{Applications of our techniques.}
Our techniques can be used in order to produce private consistent learners for a wide range of problems, e.g.\ regression problems. This can be done by using our private density estimator to obtain a privacy preserving density function, sampling a dataset from this (private) density function, and running a non-private learner on this dataset. We now elaborate on a concrete application in the context of {\em semi-supervised} learning \citep{vapnik_uniform_1971}.

Semi-supervised learning is a learning model in which the focus is on the sample complexity of {\em labeled} examples whereas unlabeled examples are of a significantly lower cost. Consider, for example, a hospital
conducting a study on a new disease. The hospital may already possess background
information about individuals and hence can access a large pool of unlabeled examples, but in order to label an example, an actual medical test is needed. In such scenarios it makes a lot of sense to try and use a combination of both labeled and unlabeled examples in order to reduce the required amount of labeled data.

Existing work on differentially private semi-supervised learning focused on the task of reducing the number of needed labeled examples, while keeping the unlabeled sample complexity not much larger than what is needed in the private PAC learning model. In particular, there are cases where this unlabeled sample complexity must be arbitrarily large, and there are simple cases (e.g., one-dimensional linear classifiers over $\R$) where differentially private semi-supervised learning is impossible in the PAC model.

Our techniques allow us to circumvent these impossibility results in the UC model. 
Specifically, we prove that any VC class $C$ over $\R^d$ can be privately learned in a semi-supervised setting with a near-optimal \emph{labeled} sample complexity of $\tilde O({\rm VC}(C))$, and with an unlabeled sample complexity that can depend on the target distribution.

\subsection{Related work}
\label{sec:related-work}
We are not the first to study private learning in a distribution-dependent context. However, to the best of our knowledge, prior work in this vein focused on obtaining better utility guarantees under the assumption that the underlying distribution adheres to certain ``niceness'' assumptions; for example margin assumptions. That is, these works do not aim to learn under {\em any} underlying distribution like we do, only under ``nice'' distributions. For example, private learning under margin assumptions was considered by \cite{BlumDMN05,ChaudhuriHS14,BunCS20,NguyenUZ20}, and private clustering under data stability assumptions was considered by \cite{NissimRS07,NIPS2015_051e4e12,HuangL18,ShechnerSS20,CohenKMST21,FriendlyCore}. Another related work is by \cite{HaghtalabRS20} who studied smooth analysis in the context of private learning, where the input points are perturbed slightly by nature. This is equivalent to assuming that the underlying distribution is not overly concentrated on any single point, which is similar in spirit to margin assumptions.

\section{Preliminaries}\label{sec:prelims}
\subsection{Preliminaries from differential privacy}

We write $\Lap(\mu,b)$ to denote the \emph{Laplace distribution} 
with mean $\mu$ and scale $b$. When the mean is zero we will simply write $\Lap(b)$.

\begin{definition}[\citet{DMNS06}]
Let $f$ be a function mapping databases to real vectors. 
The {\em global sensitivity} of $f$ is defined as 
$\GS{f} = \max_{\distH{S}{S'}=1}\Norm{f(S)-f(S')}_1$.
\end{definition}

\begin{definition}[The Laplace mechanism \citet{DMNS06}]
Let $f$ be a function mapping databases to vectors in $\R^k$, and let $\eps$ be a privacy parameter.
Given an input database $S$, the Laplace mechanism outputs 
$\mathbb{M}_\eps(f,S) = f(S) + (a_1,\ldots,a_k)$,
when $a_i$ are sampled i.i.d.\ from $\Lap(\GS{f}/\eps)$.
\end{definition}

\begin{theorem}[\citet{DMNS06}]
\label{thm:lap}
The Laplace mechanism is $\eps$-differentially private.
\end{theorem}

One of the most fundamental statistical tasks is \emph{histogram count}. 
The task is, given a dataset, count how many times each unique datum appears in the data.
The most common private solution is the Laplace mechanism, which guarantees $(\eps,0)$-differential privacy. 
The main caveat of this approach is that the error, in some cases, might accumulate too much, since we add noise to every possible domain point.
A different technique, specified in Algorithm~\ref{alg:sbh} is to ignore zero-counts and also zero-out counts which do not exceed a certain (noisy) threshold. This allows us to avoid the above accumulation of error at the price of guaranteeing privacy with $\delta>0$.
Formally,

\begin{algorithm}
  \caption{Stability based Histogram \cite{JMLR:v20:18-549}}\label{alg:sbh}
  \begin{algorithmic}[1]
    \State Input: Dataset $\sample\in \domain^n$
    \For{$x \in \domain$}
        \If{ $\realcount{\sample}(x) = 0$ }
            \State $\dpcount(x) \leftarrow 0$
        \Else %
            \State $\dpcount(x) \leftarrow \realcount{\sample}(x) + \Lap(2/\eps)$
            \If{ $\dpcount(x) < \frac{2}{\eps}\log\left(\frac{2}{\delta}\right)+1$ }
                \State $\dpcount(x) \leftarrow 0$
            \EndIf
        \EndIf
    \EndFor    
    \State Return $\dpcount$
  \end{algorithmic}
\end{algorithm}

\begin{theorem}[\cite{JMLR:v20:18-549}]
\label{thm:sbh}
The  \algname{Stability based Histogram}  algorithm is $(\eps,\delta)$-differentially private.
Moreover, for every domain point $x\in\domain$,
the resulting count $\dpcount(x)$ is such that if $\realcount{\sample}(x) = 0$ then 
$\dpcount(x) = 0$, and otherwise $\E\abs{\dpcount(x) - \realcount{\sample}(x)} \leq O\left(\frac{1}{\eps}\cdot\min\left\{\log\frac{1}{\delta},\,\realcount{\sample}(x)\right\}\right)$.
\end{theorem}

\subsection{Preliminaries from learning theory}
Given a data distribution $\measure$ and a hypothesis $h$ we denote $\err_\measure(h) = 
\E_{(x,y)\sim \measure}[1[h(x)\neq y]].$ 
Given a data distribution $\measure$, an algorithm $\mathcal{A}$, and sample size $n$, define
\begin{align*}
\err_\measure(\mathcal{A},n) &= 
\E_{S \sim \measure^{n}}
\E_{h_S\leftarrow \mathcal{A}(S)}
[\err_\measure(h_S)] 
=
\E_{S \sim \measure^{n}}
\E_{h_S\leftarrow \mathcal{A}(S)}
\E_{(x,y)\sim \measure}[1[h_S(x)\neq y]].  
\end{align*}
In words, $\err_\measure(\mathcal{A},n)$ is the expected loss of $\mathcal{A}$ given $n$ labeled examples from $\measure$.

Given a probability measure $\measure$ over $\domain \times \{0,1\}$, we denote by $\reg$ the {\em regression function}, also known as the {\em posteriori probability function}, 
defined as
$\reg(x) = \Pr(y=1\mid x)$.
The Bayes-optimal classifier is then,
\[
h^*(x) = \begin{cases}
  1  & \reg(x) > 1/2 \\
  0 & \text{otherwise}
\end{cases},
\]
and we denote its error probability by 
$L^* = \err_\measure(h^*)$.
It can be easily shown that $h^*$ achieves the lowest error-rate among all the possible classifiers. 

\begin{definition}
An algorithm $\alg$ is said to be \emph{universally consistent} if for any distribution $\measure$ it holds that
$$
\lim_{n\to\infty}\left\{\err_\measure(\alg,n)\right\} = L^*.
$$
\end{definition}

As $\reg$ is generally unknown, a possible approach for designing universally consistent algorithms is to create an approximation $\hat{\reg}$.

\begin{definition}
Let $\hat{\reg}:\domain\rightarrow[0,1]$ be any function. 
A \emph{plug-in classification rule} w.r.t.\ $\hat{\reg}$ is defined as 
$$
\hat{h}(x) = \begin{cases}
  1  & \hat{\reg}(x) > 1/2 \\
  0 & \text{otherwise}
\end{cases}.
$$
\end{definition}

The following theorem provides a bound on the error-rate of such a construction. 
\begin{theorem}[{\citet[Theorem 2.2]{devroye2013probabilistic}}]
Let $\hat{\reg}:\domain\rightarrow[0,1]$ be any function and let $\hat{h}$ be its corresponding plug-in classification rule. Then,
\label{thm:plugin}
  $\err_\measure(\hat{h}) - \err_\measure(h^*) 
  \leq 
  2 \E\left[\abs{\reg(x)-\hat{\reg}(x)}\right],$ 
  where the expectation is over sampling $x$ from the marginal distribution on unlabeled examples from $\measure$.
\end{theorem}

In \cite{devroye2013probabilistic}, the above theorem is stated only for $\R^d$. The extension to arbitrary spaces is immediate; the details are given in Section~\ref{sec:additional} for completeness.

We now turn to our main interest, which is private learning. We define the natural definition combining privacy and consistency as follows
\begin{definition}
An algorithm $\alg$ is said to be $(\eps,\delta)$-\emph{Privately universally consistent}, or PUC for short, 
if it is $(\eps,\delta)$-differentially private and universally consistent.
\end{definition}

\begin{remark}
Note that the utility requirement and the privacy requirement in the above definition are fundamentally different: Utility is only required to hold in the limiting regime when the sample size goes to infinity. In contrast, the privacy requirement is a worst-case kind of requirement that must hold for {\em any} two neighboring inputs, no matter how they were generated, even if they were not sampled from any distribution. 
\end{remark}

\paragraph{Density estimation.}

In the problem of density estimation, given a sample containing $n$ iid (unlabeled) elements from an (unknown) underlying distribution $\measure$, our goal is to output a distribution $\hat{\measure}$ that is close (in $L_1$ distance) to the underlying distribution $\measure$. 

\begin{definition}
\label{def:density-learner}
An algorithm is said to be {\em universally consistent for density estimation in $L_1$ norm} if for any underlying distribution $\measure$ the following holds.
\begin{align*}
&\lim_{n\to\infty}
\E_{S\sim\measure^n}
\E_{\measure_n\leftarrow\alg(S)}
\| \measure - \measure_n\|_{1} 
=
\lim_{n\to\infty}
\E_{S\sim\measure^n}
\E_{\measure_n\leftarrow\alg(S)}
\int \abs{\measure(x) - \measure_n(x)}dx = 0.
\end{align*}
\end{definition}

\section{Classification}
\label{sec:classification}

\begin{algorithm}
  \caption{PCL}\label{alg:pcl}
  \begin{algorithmic}[1]
    \State Input: Sample $S_n = \{(x_i,y_i)\}_{i=1}^n$ %
    \State Set $\cubesize=\frac{1}{n^{1/(2d)}}$ %
    \State Partition the space into equally sized cubes $\partition = C_1,C_2,\dots$ with side length $\cubesize$
    \State For any $x$ denote $C(x)$ the cube s.t. $x \in C(x)$
    \State Define the hypothesis $h_{\partition}$ s.t.
    $h_{\partition}(x) = \indicator{\sum_{x_i\in C(x)} y_i + Lap(1/\eps) > \frac{|C(x)|}{2}}$ \label{alg:PCL_define_h}
    \State Return $h_{\partition}$
  \end{algorithmic}
\end{algorithm}

We begin by proving Theorem~\ref{thm:mainIntro} through the study of UC learning over the bounded Euclidean space $[0,1]^d$. Our classification algorithm is presented in Algorithm~\ref{alg:pcl}. In words: we partition the space into  equally sized cubes with side length $\cubesize$.
To classify a new point, take the bucket which it falls into
and compute a noisy majority vote within this bucket.

\begin{theorem}
Algorithm~\ref{alg:pcl} is $\eps$-differentially private.
\end{theorem}

\begin{proof}
Histogram counts as used in Algorithm~\ref{alg:pcl} have global sensitivity $1$ (see \citep{dwork2014algorithmic}). 
Hence, adding Laplace noise of scale $1/\eps$ results in $\eps$-differential privacy.
Note that although Step~\ref{alg:PCL_define_h} in the algorithm seems to access the data twice, which might require the scale of the noise to be bigger, this is not the case. To see this, notice that a different way of calculating the same majority-vote is by looking at the following sum
$\sum_{x\in C_j}(y_i - 1/2) + w_j$, where $w_j$ is the noise added to the cube $C_j$, and outputting 1 if it is greater than 0 and output 0 otherwise.
As such, this amounts to a single calculation with global sensitivity 1. 
Hence, by Theorem~\ref{thm:lap} the addition of Laplace noise of scale $1/\eps$ ensures that the noisy counts are private. As the final output is merely a post-processing of these counts, it is also $\eps$-private.
\end{proof}

\begin{theorem}
  \label{thm:const}
  Algorithm~\ref{alg:pcl} is universally-consistent. 
  
\end{theorem}

\begin{proof}[Proof of Theorem~\ref{thm:const}]

Given a test point $x\in \ourdomain^d$, denote by $A(x) =  \{X_i \in S \cap C(x) \}$ the set of points from $S$ in the same bucket with $x$,
and denote the size of that bucket as $N(x)=\Sigma_{i=1}^{n}{\indicator{X_i \in A(x)}}$. 
Also define
\begin{itemize}
    \item $\hat\reg_n(x) := \frac{1}{N(x)} \Sigma_{i:x_i \in A(x) }y_i$
    \item $\hat\reg^{\eps}_n(x) := \hat\reg_n(x) + w_j$,  
where $w_j$ is the noise added to $C(x)$.
\end{itemize}

Note, that algorithm \algname{PCL} is a {\em plug-in classifier} w.r.t.\ $\hat\reg^{\eps}_n$. Hence, by Theorem~\ref{thm:plugin}, in order to prove that it is consistent it suffices to show that
\[
  \lim_{n\to\infty}\E\left[|\hat\reg_n^\eps(x)-\reg(x) |\right] = 0.
\]

By the triangle inequality, 
$
\label{eq:classif}
  \E\big[|\hat\reg^{\eps}_n(x)-\reg(x) |\big]
  \leq 
         \E\big[|\hat\reg^{\eps}_n(x)-\hat\reg_n(x) |\big]
  + \E\big[|\hat\reg_n(x) - \reg(x) |\big].    
$
In order to show that the first term goes to zero we  use  the following lemma.
\begin{lemma}[\cite{devroye2013probabilistic}]
  \label{lem:cubes}
  For any $k \in \N$ we have 
  $\Pr[N(x)\leq k] \xrightarrow[n\to \infty]{} 0,$ 
  where the probability is over sampling $S_n\sim\measure^n$ and sampling $x\sim\measure$.
\end{lemma}

Using the above lemma, we can bound the expected gap caused by the noise as follows.
\begin{align*}
    &\E_{S,x}\E_{\alg}\big[|\hat\eta^{\eps}_n(x)-\hat\eta_n(x) |\big] \leq 
    \E_{S,x}\E_{\alg}\big[|\hat\eta^{\eps}_n(x)-\hat\eta_n(x) |\cdot\indicator{N(x)>0}\big] + \Pr[N(x)=0]
    \\
    &=\E_{S,x}\;\E_{w_j \sim\Lap}\left[ \frac{|w_j|}{N(x)}\cdot\indicator{N(x)>0} \right] + \Pr[N(x)=0] \\
    &= \E_{S,x}\left[ \frac{1}{\eps N(x)} \cdot\indicator{N(x)>0} \right] + \Pr[N(x)=0] \\
    &= 
    \frac{1}{\eps} \Bigg(\mathbb{E}\left[ \frac{1}{N(x)} \mid 0<N(x) < M \right]\cdot \Pr(0<N(x) < M) 
    + \mathbb{E}\left[ \frac{1}{N(x)} \mid N(x) \geq M\right ]\cdot \Pr(N(x) \geq M)\Bigg) \\ 
    &\quad + \Pr[N(x)=0] \leq  
    \frac{1}{\eps} \left(\Pr(N(x) < M) 
    + \frac{1}{M}
    \right). \numberthis\label{eq:noise_bound}
\end{align*}

Since this is true for every choice of $M$
and by using  Lemma~\ref{lem:cubes} again,
this also can be made arbitrarily small using sufficiently large sample size.
Hence, 
\begin{equation}
    \label{eq:noise-decay}
    \mathbb{E}\big[|\hat\eta^{\eps}_n(x)-\hat\eta_n(x) |\big] \xrightarrow[]{n\to\infty} 0.
\end{equation}

Furthermore, we recall the following result by \cite{devroye2013probabilistic}

\begin{theorem}[\cite{devroye2013probabilistic}]
  For $\cubesize$ and $n$ s.t. $\lim_{n\to\infty} \cubesize = 0$ and
  $\lim_{n\to\infty}n\cubesize^d = \infty$ 
  we get that 
  $$  \lim_{n\to\infty}\E\left[|\hat\reg_n(x)-\reg(x) |\right] = 0 .
  $$
\end{theorem}
Hence, the choice of $r=\frac{1}{n^{1/(2d)}}$, together with  \eqref{eq:noise-decay} completes the proof.
\end{proof}

As we mentioned, in the supplementary material we extend this construction to metric spaces with finite doubling dimension.

\section{Density Estimation}
\label{sec:density-estimation}
We now turn to the problem of density estimation over $\R^d$. In particular, this implies private UC learning over $\R^d$ (rather than over $[0,1]^d$ as in the previous section). 
In the same manner as in the classification task,
our algorithm works by partitioning the space into cells and approximating counts corresponding to training points in these cells. 
in the classification task, in order to obtain $\epsilon$-differential privacy, we must add noise to the counts of \emph{all} of these cells. 
This is fine in the classification setting, because we measure our error on a single test point (and “pay” in the analysis for the noise in its cell). 
In contrast, in the density estimation setting the error is accumulated on the entire space (see Definition~\ref{def:density-learner}), and we cannot accumulate errors across an unbounded number of cells. 
We overcome this using ($\eps,\delta)$-differential privacy, which allows us to add noise only to a finite number of cells. 

We now present the following histogram-based approximation algorithm for density function.

\begin{algorithm}
  \caption{PCDE}\label{alg:pcde}
  \begin{algorithmic}[1]
    \State Input: Sample $S_n = \{(x_i)\}_{i=1}^n$.
    \State Set $\cubesize=\frac{1}{n^{1/(2d)}}$ 
    \State Partition the space into equally sized cubes $\partition := C_1,C_2,\dots$ with side length $\cubesize$
    \State Apply \algname{Stability based Histogram} with input $S_n$ to obtain estimates $\hat{c}_1,\hat{c}_2,\dots$ for $c_1,c_2,\dots$, where $c_j:=|\{x\in S_n : x\in C_j\}|$ denotes the number of input points in the cube $C_j$.
    \State For $x\in C_j$ denote $c(x)=c_j$ and $\dpcount(x)=\dpcount_j$.
    \State Return the function $\hat{f}_\sample$ defined as $\hat{f}_\sample(x):= \frac{1}{n \cubesize^d}\dpcount(x)$.
  \end{algorithmic}
\end{algorithm}

\begin{remark}\label{rem:zeroout}
Due to the noises in the counts, the output $\hat{f_\sample}$ of algorithm \algname{PCDE} might not be a density function: one needs to zero out negative terms it might contain and then to normalize it. This has a negligible effect on the distance from the underlying distribution, and we ignore it for simplicity.\footnote{In more detail, let $f$ denote the target distribution, let $\hat{f}$ denote the outcome of the algorithm, and suppose that the $L_1$ distance between $f$ and $\hat{f}$ is $w$. Now let $g$ denote $\hat{f}$ after zeroing out negative terms and after normalizing it (as in Remark~\ref{rem:zeroout}). An easy calculation (follows from the triangle inequality) shows that the $L_1$ distance between $f$ and $g$ is at most $O(w)$. This means that if the $L_1$ distance between $f$ and $\hat{f}$ goes to zero, then so does the distance between $f$ and $g$.}
\end{remark}

\label{sec:privacy}
\begin{theorem}
Algorithm~\ref{alg:pcde} is $(\eps,\delta)$-differentially private.
\end{theorem}

\begin{proof}
As \algname{Stability based Histogram} is $(\eps,\delta)$-differentially private, 
and since differential privacy is closed under post-processing, the output of \algname{PCDE} is also $(\eps,\delta)$-differentially private.
\end{proof}

\begin{theorem}
\label{thm:pcde-is-const}
The output of Algorithm~\ref{alg:pcde}, 
denoted by $\hat{f}_\sample$, is 
universally consistent for density estimation in $L_1$ norm.
Namely, for every distribution $\measure$ over $\R^d$ with density function $f$ we have
\[
\lim_{n\to\infty}
\E_{\sample\sim\measure^n}
\E_{\hat{f_\sample}\leftarrow\alg(S)}
\int \abs{f(x) - \hat{f}_\sample(x)}dx  = 0.
\]
\end{theorem}

\begin{proof}
For sample $\sample$ and the corresponding partition $\partition$, 
define the classic histogram-density estimation
\begin{equation}
    \label{def:hist}
  f_\sample(x)
  := \frac{1}{n \cubesize^d} 
  \sum_{i=1}^{n}\indicator{x_i\in C(x)}.  
\end{equation}
We will be using the following theorem 
\begin{theorem}[\citet{devroye2013probabilistic}, \citet{devroye1985nonparametric}]
\label{thm:dent_hist_const}
Let $f_\sample$ denote the standard histogram estimator (defined as in \eqref{def:hist}). Then,
\[
\lim_{n\to\infty}
\E_{S\sim P^n}
\int \abs{f(x) - f_\sample(x)}dx = 0.\]
\end{theorem}

Now, by the triangle inequality,
\begin{align*}
&\E_{S\sim\measure^n}
\E_{\hat{f_\sample}\leftarrow\alg(S)}
\int \abs{f(x) - \hat{f_\sample}(x)}dx 
\leq
\E_{S\sim\measure^n}
\E_{\hat{f_\sample}\leftarrow\alg(S)}
\int \abs{f_\sample(x) - \hat{f_\sample}(x)}dx 
+
\E_{S\sim\measure^n}
\int \abs{f(x) - f_\sample(x)}dx 
\numberthis\label{eq:triangle}.
\end{align*}

Hence, by Theorem~\ref{thm:dent_hist_const}, it suffices to show that 
$
  \lim_{n\to\infty}
    \E_{S,\hat{f_\sample}}
    \int \abs{f_\sample(x) - \hat{f_\sample}(x)}dx = 0.
$  
To this end, let $\arbitrarilysmall > 0$ be some parameter,
let $\support_0$ be such that there exist a cube $\supportcube_0$ of side-length $\support_0$ satisfying  
$\measure(\supportcube_0) > 1-\arbitrarilysmall.$ Now let $\supportcube$ denote the cube $\supportcube_0$ after extending it by 1 in each direction (so $\supportcube$ is a cube of side length $\support:=\support_0+2$).

\begin{remark}
Recall that the cubes $C_j$ defined by Algorithm~\ref{alg:pcde} are of side length $r\leq1$. Thus, any cube $C_j$ that intersects $\supportcube_0$ is contained in $\supportcube$.
\end{remark}

The interior of $\supportcube$ will be partitioned into $\frac{\support^d}{r^d}$ cubes of volume $r^d$.
If we restrict our calculation to $\supportcube$, 
we get that

\begin{align*}
\E_{S\sim\measure^n}&
\E_{\hat{f_\sample}\leftarrow\alg(S)}
\int_{\supportcube} \abs{f_\sample(x) - \hat{f_\sample}(x)}dx 
=
\int_{\supportcube} 
\E_{S}
\E_{\hat{f_\sample}}
\abs{f_\sample(x) - \hat{f_\sample}(x)}dx\\
\lesssim&\int_{\supportcube} 
\frac{1}{n\cubesize^d} \cdot
\frac{1}{\eps}\log\left(\frac{1}{\delta}\right)
dx
=
\frac{1}{n\cubesize^d} \cdot
\frac{1}{\eps}\log\left(\frac{1}{\delta}\right)\int_{\supportcube} 
dx
=
\support^d
\frac{1}{n\cubesize^d} \cdot
\frac{1}{\eps}\log\left(\frac{1}{\delta}\right)
= 
\frac{\support^d}{\eps\sqrt{n}}
\log\left(\frac{1}{\delta}\right),
\numberthis\label{eq:inside}
\end{align*}
where the inequality is by Theorem~\ref{thm:sbh} (after neglecting the constant hiding in the $O$-notation)
and the last equality is by the choice of 
$\cubesize = \frac{1}{n^{1/(2d)}}.$

Outside $\supportcube$, by its definition,
we have $\measure(\bar{\supportcube}) \leq \measure(\bar{\supportcube_0}) < \arbitrarilysmall$ 
and therefore %
\begin{equation}
    \label{eq:outside-vol}
\E\left[|S\cap \bar{\supportcube}|\right]
\leq
\E\left[|S\cap \bar{\supportcube_0}|\right]
< n\arbitrarilysmall.
\end{equation}

We can calculate that
\begin{align*}
\E_{S\sim\measure^n}
&\E_{\hat{f_\sample}\leftarrow\alg(S)}
\int_{\bar{\supportcube}} 
\abs{f_\sample(x) - \hat{f_\sample}(x)}dx 
=
\E_{S}
\int_{\bar{\supportcube}}
\E_{\hat{f_\sample}}
\abs{f_\sample(x) - \hat{f_\sample}(x)}dx
\leq
\E_{S}
\int_{\bar{\supportcube}} 
\frac{1}{\eps n\cubesize^d} \cdot
c(x)\;
dx \\
&=
\frac{1}{\eps n\cubesize^d} \cdot
\E_{S}
\int_{\bar{\supportcube}} 
c(x)\;
dx
\leq
\frac{1}{\eps n\cubesize^d} \cdot
\E_{S}
\sum_{C_j : C_j\cap \bar{\supportcube}\neq\emptyset} 
|S\cap C_j|\cdot r^d
\leq
\frac{1}{\eps n\cubesize^d} \cdot
\E_{S}
\sum_{C_j : C_j\subseteq \bar{\supportcube_0}} 
|S\cap C_j|\cdot r^d \\
&\leq
\frac{1}{\eps n\cubesize^d} \cdot
\E_{S}
|S\cap \bar{\supportcube_0}|\cdot r^d
\leq
\frac{\arbitrarilysmall}{\eps}
,
\numberthis\label{eq:outside}
\end{align*}
where the first inequality follows from the properties of  \algname{Stability based Histogram}, and the last inequality follows from \eqref{eq:outside-vol}.

Finally, combining \eqref{eq:inside} and \eqref{eq:outside} yields
\begin{align*}
&\E_{S\sim\measure^n}
\E_{\hat{f_\sample}\leftarrow\alg(S)}
\int \abs{f_\sample(x) - \hat{f_\sample}(x)}dx \\
&= 
\E_{S\sim\measure^n}
\E_{\hat{f_\sample}\leftarrow\alg(S)}
\int_{\support} \abs{f_\sample(x) - \hat{f_\sample}(x)}dx 
+
\E_{S\sim\measure^n}
\E_{\hat{f_\sample}\leftarrow\alg(S)}
\int_{\bar{\support}} \abs{f_\sample(x) - \hat{f_\sample}(x)}dx \\
&\lesssim
\frac{\support^d}{\eps\sqrt{n}}
\log\left(\frac{1}{\delta}\right)
+
\frac{\arbitrarilysmall}{\eps}.
\end{align*}
As $\frac{\support^d}{\sqrt{n}} \xrightarrow[]{n\to\infty} 0$ 
and $\arbitrarilysmall$ can be arbitrarily small we get that
$$
  \lim_{n\to\infty}
    \E_{S,\hat{f_\sample}}
    \int \abs{f_\sample(x) - \hat{f_\sample}(x)}dx = 0.
$$
This completes the proof.
\end{proof}

\subsection{Consistent and Private Semi-Supervised Learning}

We next show that the above result yields an application to the setting of semi-supervised private learning. 
Let $\class$ be a class of concepts. 
Recall that in the semi-supervised setting, we are given two samples 
$\sample \in (\domain \times \{0,1\})^m$ and 
$\unlabeled \in (\domain \times \{\perp\})^{n}$. For simplicity, we will restrict our discussion in this subsection to the realizable setting. Let us first recall the definition of semi-supervised learning (SSL) in the distribution free PAC model.

\begin{definition}
An algorithm $\alg$ is said to be an \emph{SSL learning algorithm} for a class $\class$ if for every $\alpha,\beta$ 
there exist 
$m=m(\alpha,\beta,\class)$
and
$n=n(\alpha,\beta,\class)$ 
such that for every distribution $\measure$ 
it holds that 
$$
\Pr_{\sample\sim\measure^m,\unlabeled\sim\bar{\measure}^n, h\sim\alg(\sample,\unlabeled)}\left[
\err_\measure(h) > \alpha
\right] < \beta, 
$$
where $\bar{\measure}$ is the marginal distribution of the unlabeled samples.
\end{definition}

\begin{definition}[Private SSL]
An algorithm is said to be a \emph{PSSL-learning algorithm} for a class $\class$ if it is an SSL-learner for $\class$ and also it is $(\eps, \delta)$-differentially private.
\end{definition}

As in the standard learning model (where all examples are labeled), semi-supervised learning can be defined in the distribution-dependent setting, or consistent setting, as follows.

\begin{definition}
An algorithm $\alg$ is said to be a \emph{consistent semi-supervised learner} (CSSL for short) for a class $\class$ if for every $\alpha,\beta$
there exist 
$m=m(\alpha,\beta,\class)$
such that for every distribution $\measure$ there is some $n=n(\alpha,\beta,\class,\measure)$ for which 
$$
\Pr_{\sample\sim\measure^m,\unlabeled\sim\bar{\measure}^n, h\sim\alg(\sample,\unlabeled)}\left[
\err_\measure(h) > \alpha
\right] < \beta,
$$
where $\bar{\measure}$ is the marginal distribution of the unlabeled samples.
\end{definition}

Note that in the above definition, we required the labeled sample complexity to be {\em uniform} over all possible underlying distributions, while allowing the unlabeled sample complexity to depend on the underlying distribution. This is interesting because with differential privacy there are cases where semi-supervised learning cannot be done in the distribution-free setting. We show that it suffices for the unlabeled sample complexity to depend on the underlying distribution while keeping the labeled sample complexity independent of it.

\begin{definition}
An algorithm is an
$(\eps,\delta)$-\emph{private consistent semi-supervised learner} (private-CSSL for short)
if it is a consistent semi-supervised learner and $(\eps,\delta)$-differentially private.
\end{definition}

For the following result, we will be using the notion of \emph{semi-private learning}.
The notion captures a scenario in which the data is sensitive, but the underlying distribution is not. This is modeled by defining a semi-supervised learning task in which the learner is required to preserve privacy only for the labeled part of the sample. Formally, a {\em semi-private} SSL algorithm is an SSL algorithm that satisfies differential privacy w.r.t.\ its labeled database (for every fixture of its unlabeled database).
\begin{theorem}[\cite{DBLP:journals/toc/BeimelNS16,BassilyMA19}]
\label{thm:semi-private}
for any concept class $\class$,
there exists a semi-private SSL algorithm 
with labeled sample complexity 
$m = \bigO{\frac{1}{\eps\alpha}VC(\class)\log\left(\frac{1}{\alpha\beta}\right)}$ 
and unlabeled sample complexity 
$n = \bigO{\frac{1}{\alpha}VC(\class)\log\left(\frac{1}{\alpha\beta}\right)}$.
\end{theorem}

As an application of our results for density estimation, we get the following corollary.
\begin{theorem}
\label{cor:ssl}
    For every class $\class$ over $\R^d$ with $VC(\class)<\infty$ 
    and for every $\eps,\delta$, 
    there exists a proper $(\eps,\delta)$-private-CSSL for $\class$ whose (labeled) sample complexity is 
    $
    m = \bigO{\frac{1}{\eps\alpha}VC(\class)\log\left(\frac{1}{\alpha\beta}\right)}.
    $
\end{theorem}

\begin{remark}
Notice that the labeled sample complexity is optimal, as a sample of size $\bigO{\vc(\class)}$ is necessary in order to learn a concept class $\class$ even 
without the privacy requirement.
\end{remark}

\begin{proof}[Proof of Theorem~\ref{cor:ssl}]
Let $\class$ be some class with 
$\vc(\class)<\infty$. Let $\alg$ be a semi-private SSL algorithm for $\class$, as guaranteed by Theorem~\ref{thm:semi-private}, and let $m_{\rm semi}$ and $n_{\rm semi}$ denote its labeled and unlabeled sample complexities, respectively.

Now fix an underlying distribution $\measure$ and let $f$ denote its marginal distribution over unlabeled examples. 
By Theorem~\ref{thm:pcde-is-const}
there is some $n = n\left(\frac{\beta}{n_{\rm semi}},\beta,f\right)$
s.t.\ we can privately generate a function $\hat{f}$, 
which is $\frac{\beta}{n_{\rm semi}}$ close (in \emph{total variation distance}) to the density function $f$ w.p.\ $1-\beta$. We proceed with the analysis assuming that this is the case.

Let $U\sim f^{n_{\rm semi}}$ denote a sample containing $n_{\rm semi}$ samples from $f$ and let $\hat{U}\sim\hat{f}^{n_{\rm semi}}$ denote a sample containing $n_{\rm semi}$ samples from $\hat{f}$. As $f,\hat{f}$ are $\frac{\beta}{n_{\rm semi}}$ close in total variation distance, we get that $f^{n_{\rm semi}}$ and $\hat{f}^{n_{\rm semi}}$ are $\beta$ close in total variation distance. 
By Theorem~\ref{thm:semi-private} we know that
$$
\Pr_{\substack{S\sim\measure^{m_{\rm semi}},\\U\sim f^{n_{\rm semi}}\\h\leftarrow\alg(S,U)}}[\err_{\measure}(h)>\alpha]<\beta,
$$
and so,
$$
\Pr_{\substack{S\sim\measure^{m_{\rm semi}},\\\hat{U}\sim \hat{f}^{n_{\rm semi}}\\h\leftarrow\alg(S,\hat{U})}}[\err_{\measure}(h)>\alpha]<2\beta.
$$

The unlabeled sample is accessed only via the private-density estimation algorithm,
and the labeled sample is accessed only via the semi-private learning method. The algorithm is therefore differentially private by composition and post-processing.
\end{proof}

\section*{Acknowledgments}
\emph{Shay Moran} is a Robert J.\ Shillman Fellow; he acknowledges support by ISF grant 1225/20, by BSF grant 2018385, by an Azrieli Faculty Fellowship, by Israel PBC-VATAT, by the Technion Center for Machine Learning and Intelligent Systems (MLIS), and by the European Union (ERC, GENERALIZATION, 101039692). Views and opinions expressed are however those of the author(s) only and do not necessarily reflect those of the European Union or the European Research Council Executive Agency. Neither the European Union nor the granting authority can be held responsible for them. \\
\emph{Haim Kaplan} is supported in part by grant 1595/19 from the Israel Science Foundation (ISF) and the Blavatnik Family Foundation. \\
\emph{Aryeh Kontorovich} was partially supported by
the Israel Science Foundation
(grant No. 1602/19), an Amazon Research Award,
and the Ben-Gurion University Data Science Research Center. \\
\emph{Uri Stemmer} was partially supported by the Israel Science Foundation (grant 1871/19) and by Len Blavatnik
and the Blavatnik Family foundation. \\
\emph{Yishay Mansour} has received funding from the European Research Council (ERC) under the European Union’sHorizon 2020 research and innovation program (grant agreement No. 882396), by the Israel Science Foundation(grant number 993/17), the Yandex Initiative for Machine Learning at Tel Aviv University and Tel Aviv University Data-Science (TAD) Center as part of the Israel Council for Higher Education Data-Science Program.

\bibliographystyle{plainnat} 
\bibliography{ref} 

\appendix

\section{Metric Spaces with Finite Doubling Dimension}
\label{sec:doubling}
In this section, we extend our results to the more general setting of metric spaces with bounded doubling dimension. We first present some additional preliminaries.

\begin{definition}[Doubling dimension]
For a metric space $(\domain,\metric)$, let $\expddim > 0$ be the smallest integer such that every ball in $\domain$ can be covered by $\expddim$ balls of half the radius.
The \emph{doubling dimension} of $(\domain,\metric)$ is 
$\ddim(\domain) = \log_2(\expddim)$.
\end{definition}

\begin{definition}
For a metric space $(\domain,\metric)$, a set of points
$\cover$ %
in $\domain$  is said to be $\radius$-cover of $\domain$ if 
for every $x\in\domain$ there exist some $x'\in \cover$ s.t.
$\dist{x}{x'} \leq \radius.$
\end{definition}

\begin{definition}
For a metric space $(\domain,\metric)$, a set of points
$\packing$ %
in $\domain$  is said to be $\radius$-packing of $\domain$ if 
for every $x,x'\in\packing$
$\dist{x}{x'} \geq \radius.$

An $\radius$-packing $\packing$ is said to be {\em maximal} if for any $x\in\domain\setminus\packing$ it holds that $\packing\cup\{x\}$ is not an $\radius$-packing of $\domain$. 
Namely, it means that there is some $x'\in\packing$ s.t. $\dist{x}{x'} < \radius.$
\end{definition}

We will be leveraging the following classical connection between packing and covering.

\begin{theorem}[\citet{vershynin2018high}]\label{thm:coverPacking}
\phantom{blah}
\begin{enumerate}
    \item Let $\packing$ be a {\em maximal $\radius$-packing} of $\domain$,
    then $\packing$ is also an {\em $\radius$-cover} of $\domain.$
    \item If there exists an {\em $\radius$-cover} of $\domain$ of size $m$, then any {\em $2\radius$-packing} of $\domain$ is of  size at most $m$.
\end{enumerate}
\end{theorem}

\begin{definition}
A metric space $(\domain,\metric)$ is {\em separable} if it has a countable dense set. That is, there exists a {\em countable} set $Q\subseteq\domain$ such that every nonempty open subset of $\domain$ contains at least one element from $Q$.
\end{definition}

\subsection{Bounded Doubling Metric Spaces}

We begin by proving the following theorem.

\begin{theorem}\label{thm:bounded-doubling}
Let $\eps\leq1$ be a constant.
There is an $(\eps,0)$-differentially private universal consistent learner for every {\em bounded} and {\em separable} metric space with finite doubling dimension.
\end{theorem}

\begin{remark}
The separability requirement is in fact necessary. 
It have been shown by \citet{10.1214/20-AOS2029} that metric spaces which are not essentially separable has no consistent learning rules, even non-private ones.
\end{remark}

Let $(\domain,\metric)$ be a bounded and separable metric space with doubling dimension $d$.
Note that as $\domain$ has finite doubling dimension and is bounded, it has a finite covering for every $\radius$. Therefore, a {\em maximal} packing of $\domain$ will also be of finite size. 

Consider Algorithm~\ref{alg:pcl2}. The privacy properties of this algorithm are straightforward; we now proceed with its utility analysis.

\begin{algorithm}
  \caption{PCL2}\label{alg:pcl2}
  \begin{algorithmic}[1]
    \State Input: Sample $S_n = \{(x_i,y_i)\}_{i=1}^n$ 
    \State Set $\radius=\frac{1}{n^{1/(4d)}}$
    \State Let $\packing$ be an $\radius$ maximal packing of $\domain$.
    \State Partition the space into Voronoi cells centered in the elements of $\packing$: $\partition = \voronoi_1,\voronoi_2,\dots$.
    \State For any $x$ denote $\voronoi(x)$ the cell s.t. $x \in \voronoi(x)$   
    \State Define 
    $h_{\partition}(x) = \indicator{\sum_{x_i\in \voronoi(x)} y_i + Lap(1/\eps) > \frac{|\voronoi(x)|}{2}}$ \label{alg:PCL2_define_h}
    \State Return $h_{\partition}$
  \end{algorithmic}
\end{algorithm}

Given a test point $x\in \domain$, denote by $\vorsample(x) =  \{X_i \in S \cap \voronoi(x) \}$ the set of points from $S$ in the same bucket with $x$,
denote the size of that bucket as $\vorsize(x)=|\vorsample(x)|$, and lastly, $\vorsize(\voronoi):=\frac{1}{n}\Sigma_{i=1}^{n}{\indicator{X_i \in \voronoi}}$ which is the relative size of the sample points in $\voronoi$ from the entire sample. 

\begin{lemma}
\label{lem:vordiam}
    For every $\voronoi\in \partition$ it holds that $\diam{\voronoi} \leq 2\radius = \frac{2}{n^{1/(4d)}}$
\end{lemma}
\begin{proof}
Given a center point $\point_i$ from $\packing$,
denote by $\coverset_i$ the ball of radius $\radius$ around it, and by $\voronoi_i$ the Voronoi cell induced by it.
Let $a,b \in \voronoi_i$ be two points on $\voronoi_i$.
By the definition of Voronoi cells for any other center point 
$\point_j$,
it holds that
$\dist{a}{\point_j} \geq \dist{a}{\point_i}$ and $\dist{b}{\point_j} \geq \dist{b}{\point_i}$.
Therefore, if 
$\dist{a}{\point_i} > \radius$ or $\dist{b}{\point_i} > \radius$,
we will get that
$\forall \point\in\packing: \dist{a}{\point} > \radius$ or $\forall \point\in\packing: \dist{b}{\point} > \radius$,
which is a contradiction to the covering property of $\packing$.
Hence, we get that
$\dist{a}{\point_i}\leq \radius$ and $\dist{b}{\point_i}\leq \radius$ which, by the triangle inequality, result in 
$\dist{a}{b}\leq \radius$.
\end{proof}

\begin{lemma}
    \label{lem:doublingcubes}
  For any $k=k(n)$ such that $k(n)=o(n^{1/4})$ we have 
  $\Pr[\vorsize(x)\leq k] \xrightarrow[n\to \infty]{} 0,$ 
  where the probability is over sampling $S_n\sim\measure^n$ and sampling $x\sim\measure$.
\end{lemma}

\begin{remark}
This theorem (and its proof) holds for both bounded and unbounded spaces. We decided to provide it in this general form as we also make use of it to analyze the unbounded case later on.
\end{remark}

\begin{proof}
Let $\dradius = n^{1/(2d)}$, let $\ball \subseteq \domain$ be a ball of radius $\dradius$, and denote
$\bar{\ball} := \domain\setminus \ball$. 
Also let $\ball_{\rm big}$ be a ball cantered at the same point as $\ball$, but with twice the radius.

By the doubling dimension of the domain, it is possible to cover $\ball_{\rm big}$ with $\left(\frac{4\dradius}{\radius}\right)^d$ small balls each of radius $\radius/2$. 
By Theorem~\ref{thm:coverPacking}, this implies that {\em any} $r$-packing of $\ball$ is of size at most $\left(\frac{4\dradius}{\radius}\right)^d$.
In particular, $\packing\cap\ball_{\rm big}$ is of size at most $\left(\frac{4\dradius}{\radius}\right)^d$. Now observe that any Voronoi cell that intersects $\ball$ is contained in $\ball_{\rm big}$. As every such Voronoi cell corresponds to a unique point in $\packing\cap\ball_{\rm big}$, we get that there are at most $\left(\frac{4\dradius}{\radius}\right)^d$ Voronoi cell that intersects $\ball$.
As we set $\radius = \frac{1}{n^{1/(4d)}}$, this quantity equals $(4\dradius\cdot n^{1/(4d)})^d$.
{\small
\begin{align*}
&\Pr[\vorsize(x)\leq k]
\leq
\sum_{\voronoi\in\partition:\voronoi\cap \ball \neq \emptyset}\Pr(\vorsize(x)\leq k, x \in \voronoi) + \Pr(\bar{\ball}) \\
&\leq
\sum_{\voronoi\cap \ball \neq \emptyset, \Pr(\voronoi)\leq 2k/n}\Pr(\voronoi) 
+ \sum_{\substack{\voronoi\cap \ball \neq \emptyset\\ \Pr(\voronoi) > 2k/n}} Pr(\voronoi)\Pr\left(\vorsize(\voronoi)\leq \frac{k}{n}\right) + \Pr(\bar{\ball}) \\
&\leq
\frac{2k}{n}(4\dradius\cdot n^{1/(4d)})^d + \Pr(\bar{\ball}) 
+ 
\sum_{\substack{\voronoi\cap \ball \neq \emptyset\\ \Pr(\voronoi) > 2k/n}} Pr(\voronoi)
\Pr
\left(
\vorsize(\voronoi) - \E\left[\vorsize(\voronoi)\right]
\leq \frac{k}{n} - \Pr(\voronoi)
\right) \\
&\leq
\frac{2k}{n}(4\dradius\cdot n^{1/(4d)})^d + \Pr(\bar{\ball}) 
+ 
\sum_{\substack{\voronoi\cap S \neq \emptyset\\ \Pr(\voronoi) > 2k/n}} Pr(\voronoi) 
\Pr
\left(
\vorsize(\voronoi) - \E\left[\vorsize(\voronoi)\right]
\leq -\frac{\Pr(\voronoi)}{2}
\right) \numberthis\label{eq:ddimlemma}
\end{align*}
}
From this point, the proof proceeds in the same steps as in \citet[Theroem 6.2]{devroye2013probabilistic}. By Chebyshev's inequality,
\small{
\begin{align*}
\eqref{eq:ddimlemma} &\leq
\frac{2k}{n}(4\dradius\cdot n^{1/(4d)})^d + \Pr(\bar{\ball}) 
+ 
\sum_{\voronoi\cap S \neq \emptyset, \Pr(\voronoi)\geq 2k/n}
4\Pr(\voronoi)\frac{Var(\vorsize(\voronoi))}{\Pr(\voronoi)^2}\\
&\leq 
\frac{2k}{n}(4\dradius\cdot n^{1/(4d)})^d + \Pr(\bar{\ball}) 
+ 
\sum_{\voronoi\cap S \neq \emptyset, \Pr(\voronoi)\geq 2k/n}
4\Pr(\voronoi)\frac{\Pr(\voronoi)(1-\Pr(\voronoi))}{n\Pr(\voronoi)^2} 
\end{align*}
\begin{align*}
&\leq 
\frac{2k}{n}(4\dradius\cdot n^{1/(4d)})^d + \Pr(\bar{\ball}) 
+ 
\sum_{\voronoi\cap S \neq \emptyset, \Pr(\voronoi)\geq 2k/n}
4\Pr(\voronoi)\frac{\Pr(\voronoi)}{n\Pr(\voronoi)^2} \numberthis\label{eq:ddimlemma2}
\end{align*}
}
When the second inequality is due to the variance of the binomial variable $\vorsize(\voronoi)$.
\small{
\begin{align*}
\eqref{eq:ddimlemma2}
&\leq
\frac{2k+4}{n}(4\dradius\cdot n^{1/(4d)})^d + \Pr(\bar{\ball}) \\
&=
\frac{(4\dradius)^d}{n^{3/4}}(2k+4) + \Pr(\bar{\ball}) 
=
\frac{4^d}{n^{1/4}}(2k+4) + \Pr(\bar{\ball}) 
\end{align*}
}
Clearly, the first summand goes to zero when $n\to \infty$ (recall that $k=o(n^{1/4})$). As for the second summand, recall that $\dradius$ goes to $\infty$ when $n\to \infty$, and so $\Pr(\bar{\ball})$ goes to zero when $n\to \infty$.
\end{proof}

We will make use of the following theorem.

\begin{theorem}
\label{thm:plug-in-for-metric}
Given a separable metric space, a partition based classification rule is universally-consistent if 
\begin{enumerate}
    \item $\diam{\voronoi(x)} \xrightarrow[n\to \infty]{} 0$
    \item For every constant $k\in\N$ it holds that $\Pr[\vorsize(x)\leq k] \xrightarrow[n\to \infty]{} 0$
\end{enumerate}
\end{theorem}

This theorem is an extension of \citet[Theroem 6.1]{devroye2013probabilistic}, where it is stated only for $\R^d$. The proof of this theorem appears in Section~\ref{sec:additional} for completeness.

Putting it all together, we now prove the following theorem.

\begin{theorem}
\label{thm:pcl2_is_ucl}
  Algorithm~\ref{alg:pcl2} is universally-consistent. 
\end{theorem}

\begin{proof}

Define
\begin{itemize}
    \item $\hat\reg_n(x) := \frac{1}{\vorsize(x)} \Sigma_{i:x_i \in \vorsample(x) }y_i
    $
    \item $\hat\reg^{\eps}_n(x) := \hat\reg_n(x) + w_j$, where $w_j$ is the noise added to $\voronoi(x)$.
\end{itemize}

The proof is close in nature to the proof of Theorem~\ref{thm:const}.
We note, that algorithm \algname{PCL2} is a {\em plug-in classifier} w.r.t.\ $\hat\reg^{\eps}_n$. Hence, by Theorem~\ref{thm:plugin}, in order to prove that it is consistent it suffices to show that
\[
  \lim_{n\to\infty}\E\left[|\hat\reg_n^\eps(x)-\reg(x) |\right] = 0.
\]

By the triangle inequality, 
$
\label{eq:classif2}
  \E\big[|\hat\reg^{\eps}_n(x)-\reg(x) |\big]
  \leq 
         \E\big[|\hat\reg^{\eps}_n(x)-\hat\reg_n(x) |\big]
  + \E\big[|\hat\reg_n(x) - \reg(x) |\big].    
$
By the same arguments as in Theorem~\ref{thm:const} we get that
\begin{equation}
\label{eq:ddim_noise_bound}
 \E_{S,x}\E_{\alg}\big[|\hat\eta^{\eps}_n(x)-\hat\eta_n(x) |\big] \leq  
    \frac{1}{\eps} \left(\Pr(\vorsize(x) < M) 
    + \frac{1}{M}
    \right) 
\end{equation}

Since this is true for every choice of $M$
and by using  Lemma~\ref{lem:doublingcubes},
this also can be made arbitrarily small using sufficiently large sample size.
Hence, 
\begin{equation}
    \label{eq:doublin-noise-decay}
    \mathbb{E}\big[|\hat\eta^{\eps}_n(x)-\hat\eta_n(x) |\big] \xrightarrow[]{n\to\infty} 0.
\end{equation}

In order to show that 
$\lim_{n\to\infty}\E\big[|\hat\reg_n(x) - \reg(x) |\big] = 0$, by Theorem~\ref{thm:plug-in-for-metric}, it suffices to show that the following two conditions hold:
\begin{enumerate}
    \item $\diam{\voronoi(x)} \xrightarrow[n\to \infty]{} 0$
    \item $\Pr[\vorsize(x)\leq k] \xrightarrow[n\to \infty]{} 0$
\end{enumerate}
The first condition follows from Lemma~\ref{lem:vordiam} and the second condition follows from Lemma~\ref{lem:doublingcubes}.
\end{proof}

\subsection{Unbounded Doubling Metric Spaces}

We now extend the previous result to the case of {\em unbounded} doubling metric spaces. This extension comes at the cost of relaxing the privacy requirement from pure-privacy to approximated-privacy.  Formally, we show the following theorem.

\begin{theorem}\label{thm:unbounded-doubling-main}
Let $\eps\leq1$ be a constant and let $\delta:\N\rightarrow[0,1]$ be a function satisfying $\delta(n)=\omega(2^{-n^{1/4}})$.
There is an $(\eps,\delta(n))$-differentially private universal consistent learner for every separable (possibly unbounded) metric space with finite doubling dimension.
\end{theorem}

Let $(\domain,\metric)$ be a separable doubling metric space with doubling dimension $d$. Consider Algorithm~\ref{alg:pcl2b}.

\begin{algorithm}
  \caption{PCL2b}\label{alg:pcl2b}
  \begin{algorithmic}[1]
    \State Input: Sample $S_n = \{(x_i,y_i)\}_{i=1}^n$ 
    \State Set $\radius=\frac{1}{n^{1/(4d)}}$
    \State Let $\packing$ be a countable $\radius$ maximal packing of $\domain$.
    \State Partition the space into Voronoi cells centered in the elements of $\packing$: $\partition = \voronoi_1,\voronoi_2,\dots$.
    \State For any $x$ denote $\voronoi(x)$ the cell $\voronoi$ s.t.\ $x \in \voronoi$   
    \State Apply \algname{Stability based Histogram} with input $S_n$ to obtain estimates $\hat{c}_1,\hat{c}_2,\dots$ such that $\hat{c}_j\approx|\{x\in S_n : x\in \voronoi_j\}|$.   
    \State For any $x$ denote $\dpcount(x)=\dpcount_j$ such that $x\in \voronoi_j$.
    \State Apply \algname{Stability based Histogram} with input $S_n^1:=\{x\in S_n : y=1\}$ to obtain estimates $\hat{y}_1,\hat{y}_2,\dots$ such that $\hat{y}_j\approx|\{x\in S_n : y=1, x\in \voronoi_j\}|$.   
    \State For any $x$ denote $\dplabels(x)=\min\{\dplabels_j,\dpcount_j\}$ such that $x\in \voronoi_j$.    
    \State Define the hypothesis $h_{\partition}$ s.t.
    $h_{\partition}(x) = \indicator{\dplabels(x) > \frac{\dpcount(x)}{2}}$ \label{alg:PCL2b_define_h}
    \State Return $h_{\partition}$
  \end{algorithmic}
\end{algorithm}

Note that, as $\domain$ is separable, it has a countable covering and countable maximal packing for every $\radius$, and hence step $3$ is well-defined. 
\footnote{Clearly, every separable space has a countable covering. As the cardinality of a packing can be bounded by the cardinality of a cover, we get that the cardinality of every packing  must also be countable. Formally, 
given a $2\radius$-packing $\packing$ and an $\radius$-cover $\cover$,
as every $\radius$-ball centered around a point in $\cover$ contains at most one point from $\packing$, we get that there is an injection from $\packing$ to $\cover$. Hence the cardinality of $\packing$ is bounded by that of $\cover$.}
Moreover, by Theorem~\ref{thm:sbh}
the number of non-empty cells will be finite, hence the hypothesis defined at step $10$ is well-defined.

\begin{theorem}
\label{thm:unbounded-doubling}
Algorithm~\ref{alg:pcl2b} is $(2\eps,2\delta)$-differentially private.
\end{theorem}

\begin{proof}[Proof of Theorem~\ref{thm:unbounded-doubling}]
As \algname{Stability based Histogram} is $(\eps,\delta)$-differentially private, 
and since differential privacy is closed under post-processing, by standard composition theorems the output of \algname{PCL2b} is  $(2\eps,2\delta)$-differentially private.
\end{proof}

\begin{theorem}
  Algorithm~\ref{alg:pcl2b} is universally-consistent. 
\end{theorem}

\begin{proof}
Define
\begin{itemize}
    \item $\hat\reg_n(x) := \frac{1}{\vorsize(x)} \Sigma_{i:x_i \in \vorsample(x) }y_i$
    \item $\hat\reg^{\eps,\delta}_n(x) := \begin{cases}
\frac{\dplabels(x)}{\dpcount(x)} & \dpcount(x)\neq 0\\
0 & \dpcount(x) = 0 
\end{cases}.$
\end{itemize}

Most of the arguments which were made for Algorithm~\ref{alg:pcl2} in the proof of Theorem~\ref{thm:pcl2_is_ucl}
can be made also for Algorithm~\ref{alg:pcl2b}.
The only part of the proof that requires attention is to show that
$\lim_{n\to\infty}\E\big[|\hat\reg^{\eps,\delta}_n(x)-\hat\reg_n(x) |\big] = 0$. We calculate,

{\small
\begin{align*}
    &\E_{S,x,\alg}\big[|\hat\eta^{\eps,\delta}_n(x)-\hat\eta_n(x) |\big] 
    =
    \E_{S,x}\left[
    \E_{\alg}\left[
    |\hat\eta^{\eps,\delta}_n(x)-\hat\eta_n(x) |
    \right]  \right]  \\
    &\leq
    \E_{S,x}\left[
    \E_{\alg}\left[
    |\hat\eta^{\eps,\delta}_n(x)-\hat\eta_n(x) | \right]
    \cdot\indicator{N(x)>0}
      \right]+\Pr[N(x)=0]  \\
    &=
    \E_{S,x}\left[
    \E_{\alg}\left[
    \left|\frac{\dplabels(x)}{\dpcount(x)} - \frac{\sum_{i:x_i \in \vorsample(x) }y_i}{\vorsize(x)} \right|
    \right] \cdot\indicator{N(x)>0} \right]+\Pr[N(x)=0]  \\
    &\leq
    \E_{S,x}\left[
    \E_{\alg}\left[
    \left|\frac{\sum_{i:x_i \in \vorsample(x) }y_i}{\vorsize(x)}  - \frac{\dplabels(x)}{\vorsize(x)}\right|
     +
    \left|\frac{\dplabels(x)}{\vorsize(x)} - \frac{\dplabels(x)}{\dpcount(x)}\right| 
    \right] \cdot\indicator{N(x)>0}  \right] 
    +\Pr[N(x)=0]  \\
    &\leq
    \E_{S,x}\left[
    O\left(\frac{\frac{1}{\eps}\log\frac{1}{\delta}}{N(x)} \right)  
    \cdot\indicator{N(x)>0}
    \right]+\Pr[N(x)=0] \\
    &\approx
    \frac{1}{\eps}\log\frac{1}{\delta}\cdot
    \E_{S,x}\left[ \frac{1}{N(x)} \cdot\indicator{N(x)>0} \right] +\Pr[N(x)=0]
    \\
    &= \frac{1}{\eps}\log\frac{1}{\delta}\cdot
    \Bigg(
    \E\left[\left. \frac{1}{N(x)} \cdot\indicator{N(x)>0} \right|N(x)< M\right] \cdot\Pr[N(x)<M]
    + 
    \E\left[\left. \frac{1}{N(x)} \right|N(x)\geq M\right]
    \cdot \Pr[N(x) \geq M] \Bigg)\\
    &\qquad +\Pr[N(x)=0]\\
    &\leq \frac{1}{\eps}\log\frac{1}{\delta}\cdot
    \Bigg(
    \Pr[N(x)<M]+ \frac{1}{M} \Bigg) +\Pr[N(x)=0]\\
    &\leq
    \frac{2}{\eps}\log\frac{1}{\delta} \left(\Pr(N(x) < M)
    + \frac{1}{M}
    \right)   
    \numberthis\label{eq:noise_bound_b}
\end{align*}
}

Since this is true for every choice of $M$
and by using  Lemma~\ref{lem:doublingcubes},
this also can be made arbitrarily small using sufficiently large sample size
\footnote{Note that, $\delta$ can decay exponentially fast as a function of $M$ and hence also as a function of $n$, allowing the same $\delta(n)=\omega(2^{-\sqrt{n}})$ dependency as in Theorem~\ref{thm:introDE}}.
Hence, 
\begin{equation}
    \label{eq:doublin-noise-decay-b}
    \mathbb{E}\big[|\hat\eta^{\eps,\delta}_n(x)-\hat\eta_n(x) |\big] \xrightarrow[]{n\to\infty} 0.
\end{equation}
\end{proof}

\begin{remark}
Unlike our results for the (unbounded) euclidean case, where we showed a construction for a {\em density estimator}, for (unbounded) metric spaces with finite doubling dimension we only show a {\em learner}. The reason is that in our construction of a density estimator for the euclidean case we needed to compute {\em volumes} of the cells in the partition. In general metric spaces, however, we do not have a canonical analogue for the volume of a cell.
\end{remark}

\section{Additional Details for Completeness}\label{sec:additional}

The proofs provided in this section are taken from \cite{devroye2013probabilistic}. We include them here for completeness, as in \cite{devroye2013probabilistic} these theorems are stated only for $\R^d$.

\begin{theorem}
\label{thm:our-plug-in-rule}
For a probability space $(\domain, \measure)$, 
let $\hat{\reg}:\domain\rightarrow[0,1]$ be any function,
and let $\hat{h}$ be the \emph{plug-in classification rule} w.r.t.\ $\hat{\reg}$. 
Then the following holds
\[
\Pr_{(X,Y)}[\hat{h}(X)\neq Y] - L^* \leq 
2 \E\left[\abs{\reg(X)-\hat{\reg}(X)}\right]
\]
\end{theorem}

\begin{proof}[Proof of Theorem~\ref{thm:our-plug-in-rule}]
Given $x\in\domain$,
if $h^*(x)=\hat{h}(x)$,
then
\[
\Pr_{(X,Y)}[\hat{h}(X)\neq Y \mid X=x]
=
\Pr_{(X,Y)}[h^*(X)\neq Y \mid X=x].
\]
On the other hand if
$h^*(x)\neq\hat{h}(x)$,
then
\[
\abs{\reg(X)-\hat{\reg}(X)} 
\geq
\abs{\reg(X)-\frac{1}{2}}.
\]
Therefore
\begin{align*}
    &\Pr_{(X,Y)}[h^*(X)\neq Y \mid X=x]
    -
    \Pr_{(X,Y)}[\hat{h}(X)\neq Y \mid X=x] \\
    &=(2\reg(x)-1)
    (\indicator{h^*(X)=1}-\indicator{\hat{h}(x)=1}) 
    =\abs{2\reg(x)-1}
    \cdot\indicator{h^*(X)\neq \hat{h}(x)} 
\end{align*}
By the law of total probability
\begin{align*}
    &\Pr_{(X,Y)}[\hat{h}(X)\neq Y] - L^* \\
    &=
    \int_{x\in\domain}
    \Pr_{(X,Y)}[h^*(X)\neq Y \mid X=x] 
    -
    \Pr_{(X,Y)}[\hat{h}(X)\neq Y \mid X=x]
    dx \\
    &=
    \int_{x\in\domain} \abs{2\reg(x)-1}
    \cdot\indicator{h^*(X)\neq \hat{h}(x)} \measure(x) dx \\
    &=\int_{x\in\domain} 2\abs{\reg(x)-\frac{1}{2}}
    \cdot\indicator{h^*(X)\neq \hat{h}(x)} \measure(x) dx \\ 
    &= 
    \E\left[2\abs{\reg(X)-\frac{1}{2}}\cdot\indicator{h^*(X)\neq\hat{h}(X)}\right] 
    \leq 
    2 \E\left[\abs{\reg(X)-\hat{\reg}(X)}\right]
\end{align*}
\end{proof}

\begin{proof}[Proof of Theorem~\ref{thm:plug-in-for-metric}]
As any partition rule is a special case of a plug-in estimator,
we need to show that
$$\E\left[\abs{\reg(x)-\hat{\reg}_n(x)}\right] \xrightarrow[n\to \infty]{}0.$$
Define $\bar{\reg}(x) := \E[\reg(z) \mid z\in \voronoi(x)].$
By the triangle inequality
\begin{align*}
    \E\left[\abs{\reg(x)-\hat{\reg}_n(x)}\right]
    \leq 
    \E\left[\abs{\reg(x)-\bar{\reg}(x)}\right]
    +
    \E\left[\abs{\bar{\reg}(x)-\hat{\reg}_n(x)}\right]
\end{align*}
Examine the random variable $\vorsize(x)\hat{\reg}_n(x)$, 
which is the number of labeled-one points falling in the same "bucket" as $x$.
By conditioning upon which points fall in this bucket, %
the remaining randomness in this r.v. is only which of them will be labeled one. 
This is then simply a binomial random variable with "success" probability $\bar{\reg}(x)$ and $\vorsize(x)$ trials.
Thus, 
\small{
\begin{align*}
    &\E\Bigg[\abs{\bar{\reg}(x)-\hat{\reg}_n(x)} \mid \indicator{x_1 \in \voronoi(X) },\ldots, \indicator{x_n \in \voronoi(X) }\Bigg] \\
    &\leq
    \E\Bigg[\abs{\frac{\vorsize(x)\hat{\reg}_n(x)}{\vorsize(x)}-\bar{\reg}(x)}
    \mid \indicator{\vorsize(x) > 0}  , \indicator{x_1 \in \voronoi(X) },\ldots, \indicator{x_n \in \voronoi(X) }\Bigg] \\
    &\leq
    \E\Bigg[\left(\frac{\vorsize(x)\hat{\reg}_n(x)}{\vorsize(x)}-\bar{\reg}(x)\right)^2
    \mid \indicator{\vorsize(x) > 0}  , \indicator{x_1 \in \voronoi(X) },\ldots, \indicator{x_n \in \voronoi(X) }\Bigg]^{1/2} \\    
    &\leq 
    \E\Bigg[\frac{\bar{\reg}(x)(1-\bar{\reg}(x))}{\vorsize(x)} \indicator{\vorsize(x) > 0} 
    \mid \indicator{x_1 \in \voronoi(X) },\ldots, \indicator{x_n \in \voronoi(X) }\Bigg]^{1/2} \numberthis\label{eq:proof-from-a-book}
\end{align*}
}
When the second inequality is by the  Jensen inequality and the third by the variance of a binomial distribution.
Next, note that $\bar{\reg}(x)(1-\bar{\reg}(x))\leq\frac{1}{4}$ and hence,
\begin{align*}
    &\eqref{eq:proof-from-a-book}    
    \leq
    \E\left[\frac{1}{4\vorsize(X)} \mid \vorsize(X) > 0\right]^{1/2}\Pr[\vorsize(n) > 0] 
    + \Pr[\vorsize(X) = 0] \\
    &\leq 
    \E\left[\frac{1}{4\vorsize(X)} \mid \vorsize(X) > 0\right]^{1/2} \Pr[\vorsize(n) > 0] 
    + \Pr[\vorsize(X) = 0] \\
    &\leq 
    \frac{1}{2}\Pr[\vorsize(x) \leq k] + \frac{1}{2\sqrt{k}} + \Pr[\vorsize(X) = 0] \numberthis\label{eq:proof-from-a-book2}
\end{align*}
This is true for any k. Therefore \eqref{eq:proof-from-a-book2} can be made arbitrarily small by choosing k large enough and then by condition $(2)$ in the theorem's conditions.

Moving on to the first summand.
For any $\tau > 0$,
there exist a uniform continuous real-valued function $\reg_\tau$,
such that
$$\E\left[\abs{\reg(x)-\reg_\tau(x)}\right] < \tau.$$
Such function exist since for a separable metric space, 
the set of uniformly continuous, real valued, functions is dense, in $\ell_1$ norm, in the set of all  continuous, real valued, functions.
Define
$\bar{\reg}_\tau(x) := \E[\reg_\tau(z) \mid z\in \voronoi(x)]$
and by the triangle inequality,
\begin{align*}
    &\E\left[\abs{\reg(x)-\bar{\reg}(x)}\right]
    \leq   
    \E\left[\abs{\reg(x)-\reg_\tau(x)}\right]
    +
    \E\left[\abs{\reg_\tau(x)-\bar{\reg}_\tau(x)}\right] 
    +
    \E\left[\abs{\bar{\reg}_\tau(x)-\bar{\reg}(x)}\right] 
     =: (*) + (**) + (***).
\end{align*}
By the choice of $\reg_\tau(x)$, the $(*) \leq \tau$.
Also, by the definitions for $\bar{\reg}$ and $\bar{\reg}_\tau(x)$
the $(***) \leq (*) \leq \tau$.
finally, s $\reg_\tau(x)$ is uniformly continuous, there exist some $\theta$ 
s.t. the difference between points which are $\theta$-close is bounded by $\tau$. 
Hence, we get that
$(**) \leq \tau + \Pr(\diam{\voronoi(x)} > \theta)$, when by condition (1) of the theorem's conditions, can be made less than $\tau$ for large enough $n$.
All for all we showed that for any given $\tau$ we can ensure that $\E\left[\abs{\reg(x)-\bar{\reg}(x)}\right] < \tau,$ for large enough $n$.
\end{proof}

\end{document}